\documentclass{article}

\usepackage[final, nonatbib]{nips_2016}

\usepackage{amsmath,amsfonts,amsthm,amssymb}

\usepackage{graphics, graphicx, xcolor}
\usepackage{hyperref}
\usepackage{url}
\usepackage{latexsym}
\usepackage{array}
\usepackage{verbatim}
\usepackage[shortlabels,]{enumitem}
\usepackage[algo2e]{algorithm2e}
\usepackage{multirow}

\usepackage{wrapfig}

\usepackage{rotating}

\usepackage{booktabs}
\usepackage{cellspace}
\usepackage{cleveref}

\newtheorem{theorem}{Theorem}%[section]
\newtheorem{lemma}[theorem]{Lemma}
\newtheorem{proposition}[theorem]{Proposition}

\newtheorem{observation}[theorem]{Observation}
\newtheorem{definition}[theorem]{Definition}

\newtheorem{assumption}{Assumption}

\newcommand{\va}{\mathbf{a}}

\newcommand{\vc}{\mathbf{c}}

\newcommand{\vF}{\mathbf{F}} 
  
\newcommand{\vh}{\mathbf{h}}
\newcommand{\vx}{\mathbf{x}}
\newcommand{\vb}{\mathbf{b}}

\newcommand{\vg}{\mathbf{g}}

\newcommand{\vr}{\mathbf{r}}

\newcommand{\vS}{\mathbf{S}}

\newcommand{\vv}{\mathbf{v}}

\newcommand{\vz}{\mathbf{z}}

\newcommand{\vzero}{\mathbf{0}}

\DeclareMathOperator*{\argmin}{arg\,min}

\DeclareMathOperator{\sgn}{sgn}

\newcommand{\RR}{\mathbb{R}}      % Real numbers
\newcommand{\vnorm}[1]{\left\lVert#1\right\rVert} % vector norm
\newcommand{\ifn}{\mathbf{1}} % indicator function for sets
\newcommand{\abs}[1]{\left| #1 \right|}

\newcommand{\deriv}[2]{\frac {d \left[ #1 \right]} {d #2}}

\newcommand{\cH}{\mathcal{H}}

\newcommand{\lrp}[1]{\left(#1\right)}
\newcommand{\lrb}[1]{\left[#1\right]}
\newcommand{\lrsetb}[1]{\left\{#1\right\}}

\newcommand{\authcmt}[2]{\textcolor{#1}{#2}}
\newcommand{\akshay}[1]{\authcmt{red}{[AB: #1]}}

\makeatletter
\def\blfootnote{\xdef\@thefnmark{}\@footnotetext}
\makeatother

\title{
Optimal Binary Classifier Aggregation for General Losses
}

\author{
  Akshay Balsubramani \\
  University of California, San Diego\\
  \url{abalsubr@ucsd.edu}
  \And
  Yoav Freund \\
  University of California, San Diego\\
  \url{yfreund@ucsd.edu}
}

\begin{document}

\maketitle

\begin{abstract}
We address the problem of aggregating an ensemble of predictors with known loss bounds in a semi-supervised binary classification setting, to minimize prediction loss incurred on the unlabeled data. We find the minimax optimal predictions for a very general class of loss functions including all convex and many non-convex losses, extending a recent analysis of the problem for misclassification error. The result is a family of semi-supervised ensemble aggregation algorithms which are as efficient as linear learning by convex optimization, but are minimax optimal without any relaxations. Their decision rules take a form familiar in decision theory -- applying sigmoid functions to a notion of ensemble margin -- without the assumptions typically made in margin-based learning. 
\end{abstract}

\section{Introduction}

Consider a binary classification problem, 
%in which we attempt to build the most accurate predictor possible for data falling into two classes. 
%At our disposal is an ensemble of individual classifiers to be used in designing our predictor. 
in which we are given an ensemble of individual classifiers to aggregate into the most accurate predictor possible for data falling into two classes. 
Our predictions are measured on a large test set of unlabeled data, 
on which we know the ensemble classifiers' predictions but not the true test labels. 
Without using the unlabeled data, the prototypical supervised solution is empirical risk minimization (ERM): 
measure the errors of the ensemble classifiers with labeled data, and then simply predict according to the best classifier. 
But can we learn a better predictor by using unlabeled data as well?

This problem is central to semi-supervised learning. 
The authors of this paper recently derived the worst-case-optimal solution for it when performance is measured with classification error (\cite{BF15}). 
However, this zero-one loss is inappropriate for other common binary classification tasks, 
such as estimating label probabilities, and handling false positives and false negatives differently. 
Such goals motivate the use of different evaluation losses like log loss and cost-weighted misclassification loss. 

In this paper, we generalize the setup of \cite{BF15} to these loss functions and a large class of others. 
Like the earlier work, the choice of loss function completely specifies the minimax optimal
ensemble aggregation algorithm in our setting, which is efficient and scalable. 

The algorithm learns weights over the ensemble classifiers by minimizing a convex function. 
The optimal prediction on each example in the test set is a sigmoid-like function of a linear combination of the ensemble predictions, using the learned weighting. 
Due to the minimax structure, this decision rule depends solely upon the loss function 
and upon the structure of the ensemble predictions on data, with no parameter or model choices. 
%It also establishes the minimax optimal prediction to have structure reminiscent of a weighted majority vote over the ensemble, and exactly paralleling the prediction function of a generalized linear model (\cite{MN89}). 

\subsection{Preliminaries}

Our setting generalizes that of \cite{BF15}, 
in which we are given an ensemble $\cH = \{ h_1, \dots, h_p \}$ and unlabeled (test) examples $x_1, \dots, x_n$ on which to predict. 
The ensemble's predictions on the unlabeled data are written as a matrix $\vF$:
\begin{equation}
\label{eq:defoff}
\vF = 
 \begin{pmatrix}
   h_1(x_1) & h_1(x_2) & \cdots & h_1 (x_n) \\
   \vdots   & \vdots    & \ddots &  \vdots  \\
   h_p(x_1)  &  h_p (x_2)  & \cdots &  h_p (x_n)
 \end{pmatrix}
% \in \RR^{p \times n}
\end{equation}
We use vector notation for the rows and columns of $\vF$: 
$\vh_i = (h_i(x_1), \cdots, h_i (x_n))^\top$ and $\vx_j =
(h_1(x_j), \cdots, h_p (x_j))^\top$. 
Each example $j \in [n]$ has a binary label $y_j \in \{-1,1\}$,  
%As in \cite{BF15}, though, t
but the test labels are allowed to be randomized, 
represented by values in $[-1,1]$ instead of just the two values $\{ -1, 1\}$; 
e.g. $z_i = \frac{1}{2}$ indicates $y_i = +1\;\text{w.p.}\; \frac{3}{4} $ and $-1\; \text{w.p.}\; \frac{1}{4} $. 
So the labels on the test data can be represented by $\vz = (z_1; \dots; z_n) \in [-1,1]^n$, 
and are unknown to the predictor, which predicts $\vg = (g_1; \dots; g_n) \in [-1,1]^n$.

\subsection{Loss Functions}

We incur loss on test example $j$ according to its true label $y_j$. 
If $y_j = 1$, then the loss of predicting $g_j \in [-1,1]$ on it is some function $\ell_{+} (g_j)$; 
and if $y_j = -1$, then the loss is $\ell_{-} (g_j)$. 
To illustrate, if the loss is the expected classification error, 
then $g_j \in [-1,1]$ can be interpreted as a randomized binary prediction in the same way as $z_j$, so that 
$\ell_{+} (g_j) = \frac{1}{2} \lrp{1 - g_j}$ and $\ell_{-} (g_j) = \frac{1}{2} \lrp{1 + g_j}$. 
%this is $\{0,1\}$ when $g_j = \{1, -1\}$ respectively, and a convex combination of the two when $g_j \in (-1,1)$. 

We call $\ell_{\pm}$ the \emph{partial losses} here, following earlier work (e.g. \cite{RW10}). 
Since the true label can only be $\pm 1$, the partial losses fully specify the decision-theoretic problem we face, 
and changing them is tantamount to altering the prediction task. 

What could such partial losses conceivably look like in general? 
Observe that they intuitively measure discrepancy to the true label $\pm 1$. 
Consequently, it is natural for e.g. $\ell_{+} (g)$ to be decreasing, as $g$ increases toward the notional true label $+1$. 
This suggests that both partial losses $\ell_{+} (\cdot)$ and $\ell_{-} (\cdot)$ would be monotonic, which we assume hereafter in this paper 
(throughout we use \emph{increasing} to mean ``monotonically nondecreasing" and vice versa). 
\begin{assumption}
\label{ass:loss}
Over the interval $(-1,1)$, $\ell_{+} (\cdot)$ is decreasing and $\ell_{-} (\cdot)$ is increasing, 
and both are twice differentiable.%\footnote{Differentiability is convenient for our proofs, but most of our arguments do not require it; see \Cref{sec:discussion}.}
\end{assumption}
We view Assumption \ref{ass:loss} as very mild, as motivated above.  
Notably, convexity or symmetry of the partial losses are not required. 
In this paper, ``general losses" refer to loss functions whose partial losses satisfy Assumption \ref{ass:loss}, to contrast them with convex losses or other subclasses. 
%Our results hold for all general losses. 
%We will develop parameter-free learning algorithms based on the given loss function(s) alone.

The expected loss incurred w.r.t. the \emph{randomized} true labels $z_j$ is a linear combination of the partial losses: 
%= \lrp{\frac{1+z_j}{2}} \ell (1, g_j) + \lrp{\frac{1-z_j}{2}} \ell (-1, g_j) 
\begin{align}
\label{eq:defofexploss}
\ell (z_j, g_j) := \lrp{\frac{1+z_j}{2}} \ell_{+} (g_j) + \lrp{\frac{1-z_j}{2}} \ell_{-} (g_j)
\end{align}
%where we conveniently write $\ell_{+} (g_j) := \ell (1, g_j)$ and $\ell_{-} (g_j) := \ell (-1, g_j)$. 

Decision theory and learning theory have thoroughly investigated the nature of the loss $\ell$ and its partial losses, 
particularly how to estimate the ``conditional label probability" $z_j$ using $\ell (z_j, g_j)$. 
A natural operation to do this is to minimize the loss over $g_j$, 
and a loss $\ell$ such that $\displaystyle \argmin_{g \in [-1,1]} \ell (z_j, g) = z_j$ (for all $z_j \in [-1,1]$) 
is called a \emph{proper loss} (\cite{S71, RW10}).%, which will be used in later discussions.

%----------------------------------------------------------------------------------------------------------------------------------------------------------------------------------------------------------------------------------
%----------------------------------------------------------------------------------------------------------------------------------------------------------------------------------------------------------------------------------

\subsection{Minimax Formulation}
\label{sec:zerooneconstr}

%\subsection{Formulation as a Zero-Sum Game}
%\label{sec:game1}

As in \cite{BF15}, we formulate the ensemble aggregation problem as a two-player zero-sum game between a predictor and an adversary.
In this game, the first player is the predictor, playing predictions over the test set $\vg \in [-1,1]^n$. 
The adversary then sets the true labels $\vz \in [-1,1]^n$. 

%The key idea is that when any ensemble constituent $i \in [p]$ is known to perform to a certain degree on the test data, 
%its predictions $\vh_i$ on the test data are a reasonable guide to $\vz$, 
%and correspondingly give us information by constraining $\vz$. 

The key idea is that any ensemble constituent $i \in [p]$ known to have low loss on the test data 
gives us information about the unknown $\vz$, as $\vz$ is constrained to be ``close" to the test predictions $\vh_i$. 
Each hypothesis in the ensemble represents such a constraint, and $\vz$ is in the intersection of all these constraint sets, 
which interact in ways that depend on the ensemble predictions $\vF$.

Accordingly, for now assume the predictor knows a vector of label correlations 
$\vb$ such that
\begin{align}
\label{eq:errconstr}
\forall i \in [p] : \qquad \frac{1}{n} \sum_{j=1}^n h_i (x_j) z_j \geq b_i 
\end{align}
i.e. $ \frac{1}{n} \vF \vz \geq \vb$. 
When the ensemble is composed of binary classifiers which predict in $[-1,1]$, 
these $p$ inequalities represent upper bounds on individual classifier error rates. 
These can be estimated from the training set w.h.p. when the training and test data are i.i.d. using uniform convergence, 
exactly as in the prototypical supervised ERM procedure discussed in the introduction (\cite{BEHW87}).
So in our game-theoretic formulation, 
the adversary plays under ensemble constraints defined by $\vb$. 

The predictor's goal is to 
\emph{minimize the worst-case expected loss of $\vg$ on the test data} 
(w.r.t. the randomized labeling $\vz$), using the loss function as defined earlier in \Cref{eq:defofexploss}: 
$$ \ell (\vz, \vg) := \frac{1}{n} \sum_{j=1}^{n} \ell (z_j, g_j) $$
This goal can be written as the following optimization problem, a two-player zero-sum game:
\begin{align}
\label{game1eq} 
V &:= \min_{\vg \in [-1,1]^n} \; \max_{\substack{ \vz \in [-1,1]^n , \\ \frac{1}{n} \vF \vz \geq \vb }} \; \ell (\vz, \vg) \\
\label{eq:linearinz} 
&= \min_{\vg \in [-1,1]^n} \; \max_{\substack{ \vz \in [-1,1]^n , \\ \frac{1}{n} \vF \vz \geq \vb }} \;\; 
\frac{1}{n} \sum_{j=1}^{n} \lrb{ \lrp{\frac{1+z_j}{2}} \ell_{+} (g_j) + \lrp{\frac{1-z_j}{2}} \ell_{-} (g_j) }
\end{align}

In this paper, we solve the learning problem faced by the predictor, 
finding an optimal strategy $\vg^*$ realizing the minimum in \eqref{game1eq} for any given ``general loss" $\ell$. 
This strategy guarantees the best possible worst-case performance on the unlabeled dataset, 
with an upper bound of $V$ on the loss. 
%This bound is perfectly tight \akshay{Justify by Sion}, by virtue of the argument above -- 
Indeed, for all $\vz_0$ and $\vg_0$ obeying the constraints, 
\Cref{game1eq} implies the tight inequalities
\begin{align} 
\label{eq:weakduality}
\min_{\vg \in [-1,1]^n} \; \ell (\vz_0, \vg) 
\stackrel{(a)}{\leq} V 
\leq \max_{\substack{ \vz \in [-1,1]^n , \\ \frac{1}{n} \vF \vz \geq \vb }} \; \ell (\vz, \vg_0)
\end{align}
and $\vg^*$ attains the equality in $(a)$, with a worst-case loss as good as \emph{any} aggregated predictor.

In our formulation of the problem, the constraints on the adversary take a central role. 
As discussed in previous work with this formulation (\cite{BF15, BF15b}), 
these constraints encode the information we have about the true labels. 
Without them, the adversary would find it optimal to trivially guarantee error (arbitrarily close to) $\frac{1}{2}$ 
by simply setting all labels uniformly at random ($\vz = \vzero^{n}$). 
It is clear that adding information through more constraints will never raise the error bound $V$. 
\footnote{However, it may pose difficulties in estimating $\vb$ by applying uniform convergence over a larger $\cH$ (\cite{BF15b}).}

Nothing has yet been assumed about $\ell (\vz, \vg)$ other than Assumption \ref{ass:loss}. 
Our main results will require only this, holding for general losses. 
This brings us to this paper's contributions:
\begin{enumerate}
\item
We give the exact minimax $\vg^* \in [-1,1]^n$ for general losses (\Cref{sec:gamesoln}). 
The optimal prediction on each example $j$ is a sigmoid function of a fixed linear combination of the ensemble's $p$ predictions on it, 
so $\vg^*$ is a non-convex function of the ensemble predictions. %on $x_j$. 
By \eqref{eq:weakduality}, this incurs the lowest worst-case loss of any predictor constructed with the ensemble information $\vF$ and $\vb$. 
%(Appendix \ref{sec:beatingbest} illustrates this further).
\item
We derive an efficient algorithm for learning $\vg^*$, by 
solving a $p$-dimensional convex optimization problem. 
This applies to a broad class of losses (cf. Lem. \ref{lem:cvxpotential}), 
including any with convex partial losses. 
Sec. \ref{sec:binclass} develops and discusses the results. 
\item
We extend the optimal $\vg^*$ and efficient learning algorithm for it, as above, to a large variety of more general ensembles and prediction scenarios (Sec. \ref{sec:extensions}), including constraints arising from general loss bounds, and ensembles of ``specialists" and heterogeneous features.
\end{enumerate}

%----------------------------------------------------------------------------------------------------------------------------------------------------------------------------------------------------------------------------------
%----------------------------------------------------------------------------------------------------------------------------------------------------------------------------------------------------------------------------------

\section{Results for Binary Classification}
\label{sec:binclass}

%A few more quantities will be convenient to specify before discussing our main results.

Based on the loss, define the function $ \Gamma : [-1,1] \mapsto \RR$ as 
$ \Gamma (g) := \ell_{-} (g) - \ell_{+} (g) $.  
(We also write the vector $\Gamma (\vg) $ componentwise with $[\Gamma (\vg)]_j = \Gamma (g_j)$ for convenience, 
so that $\Gamma (\vh_i) \in \RR^n $ and $\Gamma (\vx_j) \in \RR^p $.)
Observe that by Assumption \ref{ass:loss}, $\Gamma (g)$ is increasing on its domain; so we can discuss its inverse $\Gamma^{-1} (m)$,
%\footnote{Or pseudoinverse $\Gamma^{-1} (m) = \inf\{ g \in [-1,1] : \Gamma (g) \geq m \}$, mutatis mutandis.} %if there is no unique $\Gamma^{-1} (m)$.}
which is typically sigmoid-shaped, as will be illustrated. 

With these we will set up the solution to the game \eqref{game1eq}, which relies on a convex function.
\begin{definition}[Potential Well]
Define the \textbf{potential well}
\begin{align*}
\Psi (m) := 
\begin{cases} 
- m + 2 \ell_{-} (-1)  \qquad & \mbox{ \; if \; } m \leq \Gamma (-1) \\ 
\ell_{+} (\Gamma^{-1} (m)) + \ell_{-} (\Gamma^{-1} (m))  \qquad & \mbox{ \; if \; } m \in \lrp{ \Gamma (-1) , \Gamma (1)} \\ 
m + 2 \ell_{+} (1)  & \mbox{ \; if \; } m \geq \Gamma (1)
\end{cases}
\end{align*}
\end{definition}

\begin{lemma}
\label{lem:cvxpotential}
The potential well $\Psi (m)$ is continuous and 1-Lipschitz. 
It is also convex under \emph{any} of the following conditions:
\begin{enumerate}[(A),noitemsep]
\item
The partial losses $\ell_{\pm} (\cdot)$ are convex over $(-1,1)$.
\item
The loss function $\ell (\cdot, \cdot)$ is a proper loss. 
\item
$\ell_{-}' (x) \ell_{+}'' (x) \geq \ell_{-}'' (x) \ell_{+}' (x)$ for all $x \in (-1,1)$.
\end{enumerate}
Condition (C) is also necessary for convexity of $\Psi$, under \Cref{ass:loss}. 
\end{lemma}

%(Indeed, the proof shows that the last condition is both sufficient and necessary for convexity of $\Psi$, under \Cref{ass:loss}.) 
So the potential wells for different losses are shaped similarly, as seen in \Cref{fig:specaggalg}. 
\Cref{lem:cvxpotential} tells us that the potential well is easy to optimize under any of the given conditions. 
Note that these conditions encompass convex surrogate losses commonly used in ERM, 
including all such ``margin-based" losses (convex univariate functions of $z_j g_j$), 
introduced primarily for their favorable computational properties.% relative to direct 0-1 loss minimization. 

\begin{figure}
 \begin{minipage}[t]{.48\linewidth}
 \vspace{0pt}
%\centering
\includegraphics[width=\textwidth]{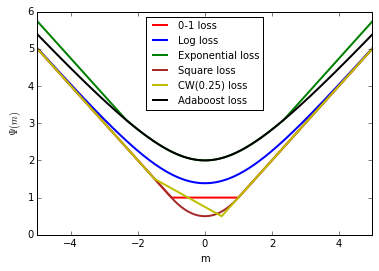}
 \end{minipage}% need this comment symbol to avoid overfull hbox
% \hspace{10pt}
  \begin{minipage}[t]{.5\linewidth}
     \vspace{0pt}
     \centering
     \includegraphics[width=\textwidth]{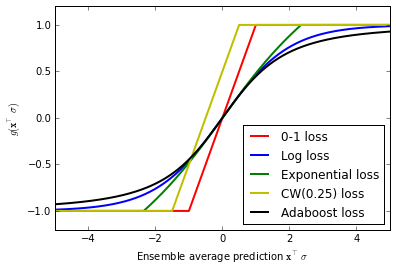}
    \end{minipage}% need this comment symbol to avoid overfull hbox
\caption{\small 
At left are plots of potential wells.
At right are optimal prediction functions $g$, as a function of score. 
Both are shown for various losses, as listed in \Cref{sec:examplelosses}. 
}
\label{fig:specaggalg}
\end{figure}

An easily optimized potential well benefits us, because the learning problem basically consists of optimizing it over the unlabeled data, as we will soon make explicit. 
The function that will actually be optimized is in terms of the dual parameters, so we call it the slack function.
%So the conditions on \Cref{lem:cvxpotential}

%As in \cite{BF15}, we show that $\vg^*$ is a simple function of a particular weighting over the $p$ hypotheses -- a non-negative $p$-vector. 
\begin{definition}[Slack Function]
Let $\sigma \geq \vzero^p$ be a weight vector over $\cH$ (not necessarily a distribution).
The vector of \textbf{scores} is
$\vF^\top \sigma = (\vx_1^\top \sigma, \dots, \vx_n^\top \sigma)$, 
whose elements' magnitudes are the \textbf{margins}. 
The prediction \textbf{slack function} is
\begin{align}
\label{eqn:slack}
\gamma (\sigma, \vb) := \gamma (\sigma) := - \vb^\top \sigma + \frac{1}{n} \sum_{j=1}^n \Psi ( \vx_{j}^\top \sigma )
\end{align}
An optimal weight vector $\sigma^*$ is any minimizer of the slack function: 
$\displaystyle \sigma^* \in \argmin_{\sigma \geq \vzero^p} \left[ \gamma (\sigma) \right]$.
\end{definition}

\subsection{Solution of the Game}
\label{sec:gamesoln}

These are used to describe the minimax equilibrium of the game \eqref{game1eq}, in our main result.
\begin{theorem}
\label{thm:gamesolngen}
The minimax value of the game \eqref{game1eq} is 
$$ \min_{\vg \in [-1,1]^n} \; \max_{\substack{ \vz \in [-1,1]^n , \\ \frac{1}{n} \vF \vz \geq \vb }} \; \ell (\vz, \vg) = V 
= \frac{1}{2} \gamma (\sigma^*) = \frac{1}{2} \min_{\sigma \geq \vzero^p} \lrb{ - \vb^\top \sigma + \frac{1}{n} \sum_{j=1}^n \Psi ( \vx_{j}^\top \sigma ) } $$
The minimax optimal predictions are defined as follows:
for all $j \in [n]$,
\begin{align}
\label{eq:gipredform}
g_j^* := g_j (\sigma^*) = 
\begin{cases} 
-1  \qquad & \mbox{ \; if \; } \vx_{j}^\top \sigma^* \leq \Gamma (-1) \\ 
\Gamma^{-1} (\vx_{j}^\top \sigma^*) \qquad & \mbox{ \; if \; } \vx_{j}^\top \sigma^* \in \lrp{ \Gamma (-1) , \Gamma (1)} \\ 
1  & \mbox{ \; if \; } \vx_{j}^\top \sigma^* \geq \Gamma (1)
\end{cases}
\end{align}
\end{theorem}

$g_j^*$ is always an increasing sigmoid, as shown in \Cref{fig:specaggalg}.

We can also redo the proof of \Cref{thm:gamesolngen} when $\vg \in [-1,1]^n$ is not left as a free variable set in the game, 
but instead is preset to $\vg (\sigma)$ as in \eqref{eq:gipredform} for some (possibly suboptimal) weight vector $\sigma$. 
\begin{observation}
\label{obs:slacksubopt}
For any weight vector $\sigma_0 \geq \vzero^p$, 
the worst-case loss after playing $\vg (\sigma_0)$ is  
$$ \max_{\substack{ \vz \in [-1,1]^n , \\ \frac{1}{n} \vF \vz \geq \vb }} \; \ell (\vz, \vg (\sigma_0) )
\leq \frac{1}{2} \gamma (\sigma_0) $$
\end{observation}
The proof is a simplified version of that of \Cref{thm:gamesolngen}; 
there is no minimum over $\vg$ to deal with, and the minimum over $\sigma \geq \vzero^p$ in \Cref{eq:gencvxineffgame} is upper-bounded 
by using $\sigma_0$. 
This result is an expression of weak duality in our setting, and generalizes Observation 4 of \cite{BF15}.
%\begin{proof}
%From the proof of Theorem \ref{thm:gamesolngen}, 
%$$ \gamma (\sigma) = - \vb^\top \sigma + \frac{1}{n} \sum_{j=1}^n \Psi ( \vx_{j}^\top \sigma ) 
%\geq \min_{\sigma' \geq \vzero^p} \left[ - \vb^\top \sigma' + \frac{1}{n} \sum_{j=1}^n \Psi ( \vx_{j}^\top \sigma' ) \right] $$
%which is simply the dual problem (Lemma \ref{lem:gamegeng}) of the worst-case loss suffered by $\vg (\sigma)$: 
%$\displaystyle  \max_{\substack{ \vz \in [-1,1]^n , \\ \frac{1}{n} \vF \vz \geq \vb }} \; \ell (\vz, \vg (\sigma) )$. 
%\end{proof}

%We discuss this more in \Cref{sec:discussion}, but first explicitly introduce the efficient learning algorithm suggested by this section's analysis.

%----------------------------------------------------------------------------------------------------------------------------------------------------------------------------------------------------------------------------------

\subsection{Ensemble Aggregation Algorithm}
\label{sec:ensaggalg}

\Cref{thm:gamesolngen} defines a prescription for aggregating the given ensemble predictions on the test set. 
%This can be stated in terms of a learning algorithm and a prediction method. 
%Our analysis implies guarantees on the behavior of both. \\

\noindent
\paragraph{Learning:}
\emph{Minimize the slack function $\gamma (\sigma)$, finding the minimizer $\sigma^*$ that achieves $V$.} \\
This is a convex optimization under broad conditions (\Cref{lem:cvxpotential}), 
and when the test examples are i.i.d. the $\Psi$ term is a sum of $n$ i.i.d. functions. 
Therefore, it is readily amenable to standard first-order optimization methods which require only $O(1)$ test examples at once. 
In practice, learning employs such methods to \emph{approximately} minimize $\gamma$, 
finding some $\sigma_A$ such that $\gamma (\sigma_A) \leq \gamma (\sigma^*) + \epsilon$ for some small $\epsilon$. 
Standard convex optimization methods are guaranteed to do this for binary classifier ensembles, 
because the slack function is Lipschitz (\Cref{lem:cvxpotential}) and $\vnorm{\vb}_{\infty} \leq 1$. 

\noindent
\paragraph{Prediction:}
\emph{Predict $g (\sigma^*)$ on any test example, as indicated in \eqref{eq:gipredform}.} \\
This decouples the prediction task over each test example separately, 
which requires $O(p)$ time and memory like $p$-dimensional linear prediction. 
After finding an $\epsilon$-approximate minimizer $\sigma_A$ in the learning step as above, 
Observation \ref{obs:slacksubopt} tells us that the prediction $\vg (\sigma_A)$ has loss $\leq V + \frac{\epsilon}{2}$.

In particular, note that there is no algorithmic dependence on $n$ in either step in a statistical learning setting. 
So though our formulation is transductive, it is no less tractable than a stochastic optimization setting in which i.i.d. data arrive one at a time, 
and applies to this common situation.

%----------------------------------------------------------------------------------------------------------------------------------------------------------------------------------------------------------------------------------
%----------------------------------------------------------------------------------------------------------------------------------------------------------------------------------------------------------------------------------

%------------------------------------------------------------------------------------------------------------------------------------------------------------------------------------------
%------------------------------------------------------------------------------------------------------------------------------------------------------------------------------------------

\subsection{Examples of Different Losses}
\label{sec:examplelosses}

To further illuminate %adumbrate
Theorem \ref{thm:gamesolngen}, we detail a few special cases in which $\ell_{\pm}$ are explicitly defined. 
These losses may be found throughout the literature (see e.g. \cite{RW10}).
The key functions $\Psi$ and $g^*$ are listed for these losses in Appendix \ref{sec:taboflosses}, and in many cases in \Cref{fig:specaggalg}. 
The nonlinearities used for $g^*$ are sigmoids, 
arising solely from the intrinsic minimax structure of the classification game.

\begin{itemize}[noitemsep]
\item
\textbf{0-1 Loss}:
Here $g_j$ is taken to be a randomized binary prediction; this case was developed in \cite{BF15}, the work we generalize in this paper.
\item
\textbf{Log Loss},  
\textbf{Square Loss}
\item
\textbf{Cost-Weighted Misclassification (Quantile) Loss}:
This is defined with a parameter $c \in [0,1]$ representing the relative cost of 
false positives vs. false negatives, making the Bayes-optimal classifier the $c$-quantile of the conditional probability distribution (\cite{S05}).
\item
\textbf{Exponential Loss}, 
\textbf{Logistic Loss}
\item
\textbf{Hellinger Loss}: This is typically given for $p, y \in [0,1]$ as $\frac{1}{2} \lrp{\lrp{\sqrt{p} - \sqrt{y}}^2 + \lrp{\sqrt{1-p} - \sqrt{1-y}}^2 }$. 
Our formulation is equivalent when the prediction and label are rescaled to $[-1,1]$.
\item
\textbf{``AdaBoost Loss"}:
If the goal of AdaBoost (\cite{SF12}) is interpreted as class probability estimation, 
the implied loss is proper and given in \cite{BSS05, RW10}.
\item
\textbf{Absolute Loss} and \textbf{Hinge Loss}:
The absolute loss can be defined by $\ell^{abs}_{\mp} (g_j) = 1 \pm g_j$, 
and the hinge loss also has $\ell^{abs}_{\mp} (g_j) = 1 \pm g_j$ since the kink in the hinge loss only lies at $g_j = \mp 1$. 
These partial losses are the same as for 0-1 loss up to scaling, and therefore all our results for $\Psi$ and $\vg^*$ are as well. 
So these losses are not shown in \Cref{tab:alllosses}. 
\item
\textbf{Sigmoid Loss}: 
This is an example of a sigmoid-shaped margin loss, a nonconvex smooth surrogate for 0-1 loss. 
Similar losses have arisen in a variety of binary classification contexts, 
from robustness (e.g. \cite{F01}) to active learning (\cite{GCB09}) and structured prediction (\cite{KM11}). 
%\item
%\textbf{Canonical Losses for CDF Inverse Link Functions}:
%There exists a unified, well-structured theory of \emph{proper losses} -- losses 
%useful for conditional probability estimation in binary classification. 
%The theory (as excellently summarized and extended by Reid and Williamson (\cite{RW10}) 
%establishes a tight correspondence between so-called \emph{link functions} and proper losses. 
%
%We briefly define these terms in our setting. 
%A link function is an increasing function $\psi$ mapping $[-1,1] \mapsto \RR$, 
%so that its inverse is well-defined. 
%We 
%\akshay{Finish this.}
\end{itemize}

\subsection{Related Work and Technical Discussion}
\label{sec:discussion}
%We reflect on some of the proof details. 
There are two notable ways in which the result of \Cref{thm:gamesolngen} is particularly advantageous and general. 
First, the fact that $\ell (z, g)$ can be non-convex in $g$, yet solvable by convex optimization, is a major departure from previous work. 
Second, the solution has a convenient dependence on $n$ (as in \cite{BF15}), simply averaging a function over the unlabeled data, which is not only 
mathematically convenient but also makes stochastic $O(1)$-space optimization practical. 
This is surprisingly powerful, because the original minimax problem is \emph{jointly} over the entire dataset, 
avoiding further independence or decoupling assumptions. 

Both these favorable properties stem from the structure of the binary classification problem, as we can describe by examining 
the optimization problem constructed within the proof of Thm. \ref{thm:gamesolngen} (\Cref{sec:proofdisc}). 
In it, the constraints which do not explicitly appear with Lagrange parameters are all box, or $L_\infty$ norm, constraints.  
These decouple over the $n$ test examples, %and the Lagrangian and loss decompose as a sum over them; 
so the problem can be reduced to the one-dimensional optimization at the heart of Eq. \eqref{eq:gencvxgame}, which is solved ad hoc. 
So we are able to obtain minimax results for these non-convex problems --  
the $g_i$ are ``clipped" sigmoid functions because of the bounding effect of the $[-1,1]$ box constraints intrinsic to binary classification. 
We introduce Lagrange parameters $\sigma$ only for the $p$ remaining constraints in the problem, 
which do not decouple as above, applying globally over the $n$ test examples. 
However, these constraints only depend on $n$ as an average over examples 
(which is how they arise in dual form in \Cref{eq:duallagrange} of the proof), 
and the loss function itself is also an average (\Cref{eq:gencvxaddloss}). 
This makes the box constraint decoupling possible, and leads to the favorable dependence on $n$, 
making an efficient solution possible to a potentially flagrantly non-convex problem. 

To summarize, the technique of optimizing only ``halfway into" the dual allows us to readily manipulate the minimax problem exactly 
without using an approximation like weak duality, despite the lack of convexity in $\vg$. 
This technique was used implicitly for a different purpose in the ``drifting game" analysis of boosting (\cite{SF12}, Sec. 13.4.1). 
Existing boosting work is loosely related to our approach in being a transductive game invoked to analyze ensemble aggregation, 
but it does not consider unlabeled data and draws its power instead from being a repeated game (\cite{SF12}).

The predecessor to this work (\cite{BF15}) addresses a problem, 0-1 loss minimization, that is known to be NP-hard when solved directly (\cite{GR09}). 
Using the unlabeled data is essential to surmounting this. %as when $n$ is large it is tantamount to knowing the distribution of . 
It gives the dual problem an independently interesting interpretation, 
so the learning problem is on the always-convex Lagrange dual function and is therefore tractable. 

This paper's transductive formulation involves no surrogates or relaxations of the loss, 
in sharp contrast to most previous work. 
This allows us to bypass the consistency and agnostic-learning discussions (\cite{Zhang04, BJM06}) common to ERM methods that use convex risk minimization. 
Convergence analyses of those methods make heavy use of convexity of the losses and are generally done presupposing a linear weighting over $\cH$ (\cite{TDS15}), 
whereas here such structure emerges directly from Lagrange duality and involves no convexity to derive the worst-case-optimal predictions.
%We explicitly conclude that the learning problem is completely determined by the choice of loss. 

The conditions in \Cref{ass:loss} are notably general. 
Differentiability of the partial losses is convenient, but not necessary, 
and only used because first-order conditions are a convenient way to establish convexity of the potential well in Lemma \ref{lem:cvxpotential}. 
It is never used elsewhere, including in the minimax arguments used to prove \Cref{thm:gamesolngen}. 
These manipulations are structured to be valid even if $\ell_{\pm}$ are non-monotonic; 
but in this case, $g_j^*$ could turn out to be multi-valued and hence not a genuine function of the example's score. 

%Also, we assert that the monotonicity condition on $\ell_{\pm}$ is natural for a loss function. 
%If it were otherwise, for instance if there were $g_1 < g_2 < g_3$ s.t. $\ell_{-} (g_1) = \ell_{-} (g_3) \neq \ell_{-} (g_2)$, then 
%this would indicate that the loss function is not simply a function of the disparity between labels; 
%in this case perhaps a different loss function is warranted. 

We emphasize that our result on the minimax equilibrium (\Cref{thm:gamesolngen}) holds for general losses -- 
the slack function may not be convex unless the further conditions of Lemma \ref{lem:cvxpotential} are met, 
but the interpretation of the optimum in terms of margins and sigmoid functions remains. 
All this emerges from the inherent decision-theoretic structure of the problem (the proof of \Cref{sec:proofdisc}). 
It manifests in the fact that the function $g (\vx_j^\top \sigma)$ is always increasing in $\vx_j^\top \sigma$ for general losses, 
because the function $\Gamma$ is increasing. 
This monotonicity typically needs to be assumed in a generalized linear model (GLM; \cite{MN89}) and related settings. 
$\Gamma$ appears loosely analogous to the link function used by GLMs, with its inverse being used for prediction. 

The optimal decision rules emerging from our framework are \emph{artificial neurons} of the ensemble inputs. 
Helmbold et al. introduce the notion of a ``matching loss" (\cite{HKW99}) for learning the parameters of a (fully supervised) artificial neuron with an arbitrary increasing transfer function, 
effectively taking the opposite tack of this paper in using a neuron's transfer function to derive a loss to minimize in order to learn the neuron's weights by convex optimization. 
Our assumptions on the loss, particularly condition (C) of Lemma 2, have arisen independently in earlier online learning work by some of the same authors (\cite{HKW98}); 
this may suggest connections between our techniques. 
We also note that our family of general losses was defined independently by \cite{S05} in the ERM setting (dubbed ``F-losses") -- 
in which condition (C) of \Cref{lem:cvxpotential} also has significance (\cite{S05}, Prop. 2) -- but has seemingly not been revisited thereafter. 
Further fleshing out the above connections would be interesting future work.

%------------------------------------------------------------------------------------------------------------------------------------------------------------------------------------------
%------------------------------------------------------------------------------------------------------------------------------------------------------------------------------------------

\section{Extensions}
\label{sec:extensions}

We detail a number of generalizations to the basic prediction scenario of Sec. \ref{sec:binclass}. 
These extensions are not mutually exclusive, and apply in conjunction with each other, but we list them separately for clarity. 
They illustrate the versatility of our minimax framework, particularly Sec. \ref{sec:specialistsgen}. 

\subsection{Weighted Test Sets and Covariate Shift}

Though our results here deal with binary classification of an unweighted test set, 
the formulation deals with a weighted set with only a slight modification of the slack function: 

%\begin{definition}[Weighted Slack Function]
%Let $\sigma \geq \vzero^p$ be a weight vector over $\cH$ and $\vr \geq 0^n$ be a weight vector over the test set. 
%Then the \textbf{weighted slack function} is
%\begin{align}
%\label{eqn:slack}
%\gamma_{\vr} (\sigma, \vb) := \gamma_{\vr} (\sigma) := - \vb^\top \sigma + \frac{1}{n} \sum_{j=1}^n r_j \Psi \lrp{ \frac{\vx_{j}^\top \sigma}{r_j} }
%\end{align}
%\end{definition}

\begin{theorem}
\label{thm:weighteval}
For any vector $\vr \geq \vzero^n$, 
\begin{align*}
\min_{\vg \in [-1,1]^n} \; \max_{\substack{ \vz \in [-1,1]^n , \\ \frac{1}{n} \vF \vz \geq \vb }} \; \frac{1}{n} \sum_{j=1}^{n} r_j \ell (z_j, g_j)
= \frac{1}{2} \min_{\sigma \geq \vzero^p} \left[ - \vb^\top \sigma + \frac{1}{n} \sum_{j=1}^n r_j \Psi \lrp{ \frac{\vx_{j}^\top \sigma}{r_j} } \right]
\end{align*}
Writing $\sigma_{\vr}^*$ as the minimizer of the RHS above, 
the optimal predictions $\vg^* = \vg (\sigma_{\vr}^*)$, as in \Cref{thm:gamesolngen}. 
\end{theorem}

Such weighted classification can be parlayed into algorithms for general supervised learning problems via learning reductions (\cite{BLZ08}). %, BHLZ10, AHKLLS14}.
Allowing weights on the test set for the evaluation is tantamount to accounting for known covariate shift in our setting; 
it would be interesting, though outside our scope, to investigate scenarios with more uncertainty.

%\akshay{
%All wrong!
%\subsection{Label Noise}
%In addition, we can vary the box constraints on $\vz$, e.g. constraining it to $-\vr \leq \vz \leq \vr$ for some $\vr \geq \vzero^n$. 
%This was done for the 0-1 loss in (\cite{BF15}, Prop. 5-6), 
%where it was interpreted as constraining the adversary to act under a level of known label noise when $\vnorm{\vr}_{\infty} \leq 1$.
%\begin{align}
%\label{eq:labnoise}
%\min_{\vg \in [-1,1]^n} \; \max_{\substack{ -\vr \leq \vz \leq \vr , \\ \frac{1}{n} \vF \vz \geq \vb }} \; \frac{1}{n} \sum_{j=1}^{n} \ell (z_j, g_j)
%= \min_{\vg \in [-1,1]^n} \; \max_{\substack{ -\vr \leq \tilde{\vz} \leq \vr , \\ \frac{1}{n} \tilde{\vF} \tilde{\vz} \geq \vb }} \; \ell(\tilde{\vz}, \vg)
%\end{align}
%where $\tilde{\vF}$ is a suitably redefined version of $\vF$ (s.t. $\tilde{\vx}_j = \frac{1}{r_j} \vx_j$). 
%The right-hand side here is formally equivalent to the original problem except for the box constraint on the adversary, which is now nonuniform.  
%Generalizing \cite{BF15}, it can be easily derived that when $\vnorm{\vr}_{\infty} \leq 1$,  
%the above Eq. \eqref{eq:labnoise} $\leq V$, 
%i.e. knowing the noise level always helps in a minimax sense by further constraining $\vz$. 
%}

\subsection{General Loss Constraints on the Ensemble}
\label{sec:genlossconstrbody}

So far in the paper, we have considered the constraints on ensemble classifiers as derived from their label correlations (i.e. 0-1 losses), 
as in \eqref{eq:errconstr}. 
However, this view can be extended significantly with the same analysis, 
because any general loss $\ell (z, g)$ is linear in $z$ (Eq. \eqref{eq:defofexploss}), 
which allows our development to go through essentially intact.

In summary, a classifier can be incorporated into our framework for aggregation if we have a generalization loss bound on it, 
for any ``general loss." 
This permits an enormous variety of constraint sets, as each classifier considered can have constraints 
corresponding to any number of loss bounds on it, 
even multiple loss bounds using different losses.
For instance, $h_1$ can yield a constraint corresponding to a zero-one loss bound, 
$h_2$ can yield one constraint corresponding to a square loss bound and another corresponding to a zero-one loss bound, 
and so on. 
\Cref{sec:genlossconstr} details this idea, extending Theorem \ref{thm:gamesolngen} to general loss constraints.

\subsection{Uniform Convergence Bounds for $\vb$}

In our basic setup, $\vb$ has been taken as a lower bound on ensemble classifier label correlations. 
But as mentioned earlier, the error in estimating $\vb$ is in fact often quantified by two-sided uniform convergence ($L_{\infty}$) bounds on $\vb$. 
Constraining $\vz$ in this fashion amounts to $L_1$ regularization of the dual vector $\sigma$. 
\begin{proposition}
\label{prop:gamegenL1}
For any $\epsilon \geq 0$, 
\begin{align*}
\min_{\vg \in [-1,1]^n} \; \max_{\substack{ \vz \in [-1,1]^n , \\ \vnorm{\frac{1}{n} \vF \vz - \vb}_{\infty} \leq \epsilon }} \;\ell (\vz, \vg)
\;=\; \min_{\sigma \in \RR^p} \left[ - \vb^\top \sigma + \frac{1}{n} \sum_{j=1}^n \Psi ( \vx_{j}^\top \sigma ) + \epsilon \vnorm{\sigma}_1 \right]
\end{align*}
%Let $\sigma_s^*$ be the minimizer of the right-hand side above.
As in Thm. \ref{thm:gamesolngen}, the optimal $\vg^* = \vg (\sigma_{\infty}^*)$, where $\sigma_{\infty}^*$ is the minimizer of the right-hand side above.
\end{proposition}
Here we optimize over all vectors $\sigma$ (not just nonnegative ones) in an $L_1$-regularized problem, 
convenient in practice because we can cross-validate over the parameter $\epsilon$. 
To our knowledge, this particular analysis has been addressed in prior work only for the special case of the entropy loss on the probability simplex, 
discussed further in \cite{DPS04}. 

Prop. \ref{prop:gamegenL1} is a corollary of a more general result using differently scaled label correlation deviations 
within the ensemble, i.e. $\abs{\frac{1}{n} \vF \vz - \vb} \leq \vc$ for a general $\vc \geq \vzero^n$. 
This turns out to be equivalent to regularizing the minimization over $\sigma$ by its $\vc$-weighted $L_1$ norm $\vc^\top \abs{\sigma}$ 
(Thm. \ref{thm:gamegendev}), used to penalize the ensemble nonuniformly (\cite{CWB08}). 
This situation is motivated by uniform convergence of heterogeneous ensembles comprised of e.g. ``specialist" predictors, 
for which a union bound (\cite{BEHW87}) results in $\abs{\frac{1}{n} \vF \vz - \vb}$ with varying coordinates. 
Such ensembles are discussed next.

%------------------------------------------------------------------------------------------------------------------------------------------------------------------------------------------

\subsection{Heterogenous Ensembles of Specialist Classifiers and Features}
\label{sec:specialistsgen}

All the results and algorithms in this paper apply in full generality to ensembles of ``specialist" classifiers  
that only predict on some subset of the test examples. 
This is done by merely calculating the constraints over only these examples, and changing $\vF$ and $\vb$ accordingly (\cite{BF15b}). 

To summarize this from \cite{BF15b}, suppose a classifier $i \in [p]$ decides to abstain on an example $x_j$ ($j \in [n]$) with probability $1 - v_i (x)$,  
and otherwise predict $h_i (x)$. 
Our only assumption on $\{v_i (x_1), \dots, v_i (x_n) \}$ is the reasonable one that  
$\sum_{j=1}^n v_i (x_j) > 0$, so classifier $i$ is not a useless specialist that abstains everywhere. 

The information about $\vz$ contributed by classifier $i$ is now not its overall correlation with $\vz$ on the entire test set, 
but rather the correlation with $\vz$ restricted to the test examples on which it predicts. 
On the test set, this is written as $ \frac{1}{n} \vS \vz$, where the matrix $\vS$ is formed by reweighting each row of $\vF$: 
\begin{align*}
\vS := n 
 \begin{pmatrix}
   \rho_1 (x_1) h_1(x_1) & \rho_1 (x_2) h_1(x_2) & \cdots & \rho_1 (x_n) h_1 (x_n) \\
   \rho_2 (x_1) h_2(x_1) & \rho_2 (x_2) h_2(x_2) & \cdots & \rho_2 (x_n) h_2 (x_n) \\
   \vdots   & \vdots    & \ddots &  \vdots  \\
   \rho_p (x_1) h_p(x_1) & \rho_p (x_2) h_p (x_2)  & \cdots & \rho_p (x_n) h_p (x_n)
 \end{pmatrix}
\quad , \quad
\rho_i (x_j) := \frac{v_i (x_j)}{ \sum_{k=1}^n v_i (x_k) }
\end{align*}
($\vS = \vF$ when the entire ensemble consists of non-specialists, recovering our initial setup.)
Therefore, the ensemble constraints \eqref{eq:errconstr} become $ \frac{1}{n} \vS \vz \geq \vb_S$, 
where $\vb_S$ gives the label correlations of each classifier restricted to the examples on which it predicts. 
Though this rescaling results in entries of $\vS$ having different ranges and magnitudes $\geq 1$, our results and proofs remain entirely intact.

%In other words, for some $[b_{S}]_{i} \in [0,1] $, 
%\begin{align}
%\label{eq:specconstr}
%\sum_{j=1}^n \lrp{ \frac{v_i (x_j)}{ \sum_{k=1}^n v_i (x_k) } } h_i (x_j) z_j \geq [b_{S}]_{i}
%\end{align}
%Now we can replace the label correlation vector $\vb$ by $\vb_S$, 
%and the unlabeled data matrix $\vF$ by the following matrix $\vS$, formed by reweighting each row (``feature"): 

Indeed, despite the title, this paper applies far more generally than to an ensemble of binary classifiers, 
because our proofs make no assumptions at all about the structure of $\vF$. 
Each predictor in the ensemble can be thought of as a feature; 
it has so far been convenient to think of it as binary, following the perspective of binary classifier aggregation, 
but it could as well be e.g. real-valued, and the features can have very different scales (as in $\vS$ above). 
An unlabeled example $\vx$ is simply a vector of features, 
so arbitrarily abstaining specialists are equivalent to ``missing features," 
which this framework handles seamlessly due to the given unlabeled data. 
Our development \emph{applies generally to semi-supervised binary classification}.

\subsection*{Acknowledgements}
AB is grateful to Chris ``Ceej" Tosh for feedback that made the manuscript clearer. 
This work was supported by the NSF (grant IIS-1162581).

%------------------------------------------------------------------------------------------------------------------------------------------------------------------------------------------

%------------------------------------------------------------------------------------------------------------------------------------------------------------------------------------------

%------------------------------------------------------------------------------------------------------------------------------------------------------------------------------------------

%------------------------------------------------------------------------------------------------------------------------------------------------------------------------------------------

\newpage
{\small
\bibliography{GenLossArxiv}{}}
\bibliographystyle{plain}

\newpage
\appendix

%------------------------------------------------------------------------------------------------------------------------------------------------------------------------------------------
%------------------------------------------------------------------------------------------------------------------------------------------------------------------------------------------

\section{Examples of Optimal Decision Rules for Various Losses}
\label{sec:taboflosses}

\begin{table}[!htbp]
\centering
\scriptsize
\rotatebox{90}{
\begin{tabular}{| Sc | Sc | Sc | Sc | Sc |} \hline \hline
Loss & Partial Losses & $\Gamma (g)$ & $\Psi (m)$ & $g_i (\sigma)$ \\ \hline \hline
 \multirow{2}{*}{0-1} & $\ell_{-} (g) = \frac{1}{2}\lrp{1 + g}$  & \multirow{2}{*}{$g$} & \multirow{2}{*}{$\max (1, \abs{m})$} 
 & \multirow{2}{*}{$\text{clip} (\vx_{i}^\top \sigma)$} \\ \cline{2-2}
 & $\ell_{+} (g) = \frac{1}{2}\lrp{1 - g}$ & & & \\ \hline
   \multirow{2}{*}{Log} & $\ell_{-} (g) = \ln \lrp{\frac{2}{1 - g}}$  & \multirow{2}{*}{$\ln \lrp{\frac{1 + g}{1 - g}}$} 
   & \multirow{2}{*}{$\ln (1 + e^{m}) + \ln (1 + e^{-m}) $} 
  & \multirow{2}{*}{$\frac{1 - e^{- \vx_{i}^\top \sigma} }{ 1 + e^{- \vx_{i}^\top \sigma} }$} \\ \cline{2-2}
 & $\ell_{+} (g) = \ln \lrp{\frac{2}{1 + g}}$ & & & \\ \hline
   \multirow{2}{*}{Square} & $\ell_{-} (g) = \lrp{\frac{1 + g}{2}}^2$  & \multirow{2}{*}{$g$} 
   & \multirow{2}{*}{$\begin{cases} 
-m  \qquad & \;\; m \leq -1 \\ 
\frac{1}{2} (m^2 + 1) \qquad & \;\; m \in \lrp{-1, 1} \\ 
m  & \;\; m \geq 1
\end{cases}$} 
   & \multirow{2}{*}{$\text{clip} (\vx_{i}^\top \sigma)$} \\ \cline{2-2}
 & $\ell_{+} (g) = \lrp{\frac{1 - g}{2}}^2$ & & & \\ \hline
   \multirow{2}{*}{CW (param. $c$)} & $\ell_{-} (g) = c \lrp{1 + g}$  & \multirow{2}{*}{$g + 2c - 1$} 
   & \multirow{2}{*}{$\begin{cases} 
-m \qquad & \;\; m \leq 2c - 2 \\ 
(2c-1)m + 4c (1-c) \qquad & \;\; m \in \lrp{2c - 2, 2c}\\
m  & \;\; m \geq 2c
\end{cases}$} 
   & \multirow{2}{*}{$\text{clip} (\vx_{i}^\top \sigma + 1 - 2c)$} \\ \cline{2-2}
 & $\ell_{+} (g) = (1-c) \lrp{1 - g}$ & & & \\ \hline
   \multirow{2}{*}{Exponential} & $\ell_{-} (g) = e^g$  & \multirow{2}{*}{$e^{g} - e^{-g}$} 
   & \multirow{2}{*}{$\begin{cases} 
-m + 2/e \qquad & \;\; m \leq - e + \frac{1}{e} \\ 
\sqrt{4 + m^2} \qquad & \;\; m \in \lrp{- e + \frac{1}{e} , e - \frac{1}{e} }\\%\;\; m \in \lrp{- e + e^{-1}, e - e^{-1}} \\ 
m + 2/e  & \;\; m \geq e - \frac{1}{e}
\end{cases}$} 
   & \multirow{2}{*}{$\text{clip} \lrp{ \ln \lrp{\frac{1}{2} \vx_{i}^\top \sigma + \sqrt{1 + \frac{1}{4} (\vx_{i}^\top \sigma)^2 } } }$} \\ \cline{2-2}
 & $\ell_{+} (g) = e^{-g}$ & & & \\ \hline
   \multirow{2}{*}{Logistic} & $\ell_{-} (g) = \ln \lrp{1 + e^g}$  & \multirow{2}{*}{$g$} 
   & \multirow{2}{*}{$\begin{cases} 
-m + 2 \ln \lrp{1 + 1/e} \qquad & \;\; m \leq - 1 \\ 
\ln (1 + e^{m}) + \ln (1 + e^{-m}) \qquad & \;\; m \in \lrp{-1, 1} \\%\;\; m \in \lrp{- e + e^{-1}, e - e^{-1}} \\ 
m + 2 \ln \lrp{1 + 1/e}  & \;\; m \geq 1
\end{cases}$} 
   & \multirow{2}{*}{$\text{clip} (\vx_{i}^\top \sigma)$} \\ \cline{2-2}
 & $\ell_{+} (g) = \ln \lrp{1 + e^{-g}}$ & & & \\ \hline
    \multirow{2}{*}{Hellinger} & $\ell_{-} (g) = 1 - \sqrt{\frac{1 - g}{2}}$  & \multirow{2}{*}{$\sqrt{\frac{1 + g}{2}} - \sqrt{\frac{1 - g}{2}}$} 
   & \multirow{2}{*}{$\begin{cases} 
-m \qquad & \;\; m \leq - 1 \\ 
2 - \sqrt{\frac{1 - m \sqrt{2 - m^2}}{2} } - \sqrt{ \frac{1 + m \sqrt{2 - m^2}}{2} }  \qquad & \;\; m \in \lrp{-1, 1} \\%\;\; m \in \lrp{- e + e^{-1}, e - e^{-1}} \\ 
m \qquad  & \;\; m \geq 1
\end{cases}$} 
   & \multirow{2}{*}{$\begin{cases} 
   (\vx_{i}^\top \sigma) \sqrt{ 2 - (\vx_{i}^\top \sigma)^2 } \qquad & \;\; | \vx_{i}^\top \sigma | \leq 1 \\
   \sgn(\vx_{i}^\top \sigma) \qquad & \;\; | \vx_{i}^\top \sigma | > 1
   \end{cases}$} \\ \cline{2-2}
 & $\ell_{+} (g) = 1 - \sqrt{\frac{1 + g}{2}}$ & & & \\ \hline
%    \multirow{2}{*}{Hinge} & $\ell_{-} (g) = 1 - \sqrt{\frac{1 - g}{2}}$  & \multirow{2}{*}{$\sqrt{\frac{1 + g}{2}} - \sqrt{\frac{1 - g}{2}}$} 
%   & \multirow{2}{*}{$\begin{cases} 
%-m \qquad & \;\; m \leq - 1 \\ 
%2 - \lrp{\frac{1 - m \sqrt{2 - m^2}}{2} }^{1/2} - \lrp{ \frac{1 + m \sqrt{2 - m^2}}{2} }^{1/2}  \qquad & \;\; m \in \lrp{-1, 1} \\%\;\; m \in \lrp{- e + e^{-1}, e - e^{-1}} \\ 
%m \qquad  & \;\; m \geq 1
%\end{cases}$} 
%   & \multirow{2}{*}{$\begin{cases} 
%   (\vx_{i}^\top \sigma) \sqrt{ 2 - (\vx_{i}^\top \sigma)^2 } \qquad & \;\; | \vx_{i}^\top \sigma | \leq 1 \\
%   \sgn(\vx_{i}^\top \sigma) \qquad & \;\; | \vx_{i}^\top \sigma | > 1
%   \end{cases}$} \\ \cline{2-2}
% & $\ell_{+} (g) = 1 - \sqrt{\frac{1 + g}{2}}$ & & & \\ \hline
   \multirow{2}{*}{``AdaBoost"} & $\ell_{-} (g) = \sqrt{\frac{1 + g}{1 - g}}$  & \multirow{2}{*}{$\frac{2 g}{\sqrt{1 - g^2}}$} 
   & \multirow{2}{*}{$\sqrt{\frac{\sqrt{m^2 + 4} + m}{\sqrt{m^2 + 4} - m}} + \sqrt{\frac{\sqrt{m^2 + 4} - m}{\sqrt{m^2 + 4} + m}}$} 
   & \multirow{2}{*}{$\frac{\vx_{i}^\top \sigma}{\sqrt{(\vx_{i}^\top \sigma)^2 + 4}}$} \\ \cline{2-2}
 & $\ell_{+} (g) = \sqrt{\frac{1 - g}{1 + g}}$ & & & \\ \hline
    \multirow{2}{*}{Sigmoid} & $\ell_{-} (g) = \frac{1}{1 + e^{-g}}$  & \multirow{2}{*}{$\frac{e^g - 1}{e^g + 1}$} 
   & \multirow{2}{*}{$\max \lrp{1 , \abs{m} + \frac{2}{1 + e} }$} 
   & \multirow{2}{*}{$\text{clip} \lrp{\ln \lrp{\frac{1 + m}{1 - m}}}$} \\ \cline{2-2}
 & $\ell_{+} (g) = \frac{1}{1 + e^{g}}$ & & & \\ \hline
%   \multirow{2}{*}{Link $=$ CDF $\Phi$} & $\ell_{-} (g) = \log \lrp{\frac{2}{1 - g}}$  & \multirow{2}{*}{$\log \lrp{\frac{1 + g}{1 - g}}$} & \multirow{2}{*}{0-1} 
%   & \multirow{2}{*}{$\begin{cases} 
%-1  \qquad & \;\; \vx_{i}^\top \sigma \leq \Gamma (-1) \\ 
%\vx_{i}^\top \sigma \qquad & \;\; \vx_{i}^\top \sigma \in \lrp{ \Gamma (-1) , \Gamma (1)} \\ 
%1  & \;\; \vx_{i}^\top \sigma \geq \Gamma (1)
%\end{cases}$} \\ \cline{2-2}
% & $\ell_{+} (g) = \frac{1-c}{2}\lrp{1 - g}$ & & & \\ \hline \hline
\end{tabular}
\label{tab:alllosses}
}
\caption{Some binary classification losses, as in Sec. \ref{sec:examplelosses}. 
For convenience, we write $\text{clip} (x) = \min(1, \max( -1, x))$.}
\end{table}

\newpage
\section{Constraints on General Losses for Binary Classification}
\label{sec:genlossconstr}

In all other sections of the paper, 
we allow the evaluation function of the game to be a general loss, 
but assume that the constraints (our information about the ensemble) are in terms of zero-one loss as written in \eqref{eq:errconstr}. 
However, here we relax that assumption, allowing each classifier $h_i$ to constrain the test labels $\vz$ 
not with the zero-one loss of $h_i$'s predictions, but rather with some other general loss. 

This is possible because when the true labels are binary, all the losses we consider are linear in $\vz$, as seen in \eqref{eq:linearinz}: 
$\displaystyle \ell (\vz, \vg) = \frac{1}{n} \sum_{j=1}^{n} \frac{1}{2} \lrb{ \ell_{+} (g_j) + \ell_{-} (g_j) } - \frac{1}{2 n} \vz^\top [\Gamma (\vg)] $. 
Accordingly, recall that $\vh_i \in [-1,1]^n$ is the vector of test predictions of classifier $h_i$. 
Suppose we have an upper bound on the generalization loss of $h_i$, i.e. $\ell (\vz, \vh_i) \leq \epsilon_i^{\ell}$. 
If we define 
$
b_i^{\ell} := \frac{1}{n} \sum_{j=1}^{n} \lrb{ \ell_{+} (h_i (x_j)) + \ell_{-} (h_i (x_j)) } - 2 \epsilon_i^{\ell}
$, 
then we can use the definition of $\ell (\vz, \vg)$ to write
\begin{align}
\label{eq:defoflbound}
\ell (\vz, \vh_i) \leq \epsilon_i^{\ell}
\qquad \iff \qquad
\frac{1}{n} \vz^\top [\Gamma (\vh_i)] \geq b_i^{\ell}
\end{align}
Now \eqref{eq:defoflbound} is a linear constraint on $\vz$, 
just like each of the error constraints earlier considered in \eqref{eq:errconstr}. 
We can derive an aggregation algorithm with constraints like \eqref{eq:defoflbound}, 
using essentially the same analysis as employed in \Cref{sec:binclass} to solve the game \eqref{game1eq}. 
As mentioned in Sec. \ref{sec:genlossconstrbody},

\subsection*{Matching Objective and Constraint Losses}
Though the ensemble constraints can be completely heterogeneous, we focus on a special case of them in the rest of this section to glean intuition. 
Suppose when each classifier contributes exactly one constraint to the problem, 
and the losses used for these constraints are all the same as each other and as the loss $\ell$ used in the objective function. 
In other words, the minimax prediction problem we now consider is
\begin{align}
\label{eq:game1eqgen} 
V^{\ell} &:= \min_{\vg \in [-1,1]^n} \; \max_{\substack{ \vz \in [-1,1]^n , \\ \forall i \in [p] : \;\ell (\vz, \vh_i) \leq \epsilon_i^{\ell} }} \; \ell (\vz, \vg)
= \min_{\vg \in [-1,1]^n} \; \max_{\substack{ \vz \in [-1,1]^n , \\ \forall i \in [p] : \;\frac{1}{n} \vz^\top [\Gamma (\vh_i)] \geq b_i^{\ell} }} \; \ell (\vz, \vg)
\end{align}

The matrix $\vF$ and the slack function from \eqref{eq:defoff} are therefore redefined:
\begin{align*}
F_{ij}^{\ell} := \Gamma (h_i (x_j)) = \ell_{-} (h_i (x_j)) - \ell_{+} (h_i (x_j)) 
\end{align*}
\begin{align*}
\gamma^{\ell} (\sigma, \vb^{\ell}) := \gamma^{\ell} (\sigma) 
:= - [\vb^{\ell}]^\top \sigma + \frac{1}{n} \sum_{j=1}^n \Psi \lrp{ [\Gamma (\vx_j) ]^\top \sigma }
\end{align*}
where $\vb^{\ell} = (b_1^{\ell} , \dots, b_p^{\ell})^\top$ and the vectors $\vx_j$ are now from the new unlabeled data matrix $F_{ij}^{\ell}$.
The game \eqref{eq:game1eqgen} is clearly of the same form as the earlier formulation \eqref{game1eq}. 
Therefore, its solution has the same structure as in Theorem \ref{thm:gamesolngen}, proved using that theorem's proof:

\begin{theorem}
\label{thm:solnmatloss}
The minimax value of the game \eqref{eq:game1eqgen} is 
$V := \frac{1}{2} \gamma^{\ell} (\sigma^{\ell*}) := \min_{\sigma \geq \vzero^p} \frac{1}{2} \gamma^{\ell} (\sigma)$. 
The minimax optimal predictions are defined as follows:
for all $j \in [n]$,
\begin{align*}
g_j^* := g_j (\sigma^*) = 
\begin{cases} 
-1  \qquad & \mbox{ \; if \; } [\Gamma (\vx_j) ]^\top \sigma^{\ell*} \leq \Gamma (-1) \\ 
\Gamma^{-1} \lrp{ [\Gamma (\vx_j) ]^\top \sigma^{\ell*} } \qquad & \mbox{ \; if \; } [\Gamma (\vx_j) ]^\top \sigma^{\ell*} \in \lrp{ \Gamma (-1) , \Gamma (1)} \\ 
1  & \mbox{ \; if \; } [\Gamma (\vx_j) ]^\top \sigma^{\ell*} \geq \Gamma (1)
\end{cases}
\end{align*}
\end{theorem}

This provides a concise characterization of how to solve the semi-supervised binary classification game for general losses. 
Though on the face of it \Cref{thm:solnmatloss} is a much stronger result than even \Cref{thm:gamesolngen}, 
we cannot overlook statistical issues. 
The loss bounds $\epsilon_i^{\ell}$ on each classifier are estimated using a uniform convergence bound over the ensemble with loss $\ell$, 
but the data now considered are not the ensemble predictions, but the predictions passed through function $\Gamma$. 
This can be impractical for losses like log loss, for which $\Gamma$ is unbounded, and therefore uniform convergence to estimate $b_i^{\ell}$ in \eqref{eq:defoflbound} is much less applicable than for 0-1 loss. 

But such issues are outside our scope here, and our constrained minimax results hold in any case given $\vb$. 
They may be useful to obtain semi-supervised learnability results for different losses from tighter statistical characterizations, 
which we consider an interesting open problem.

\section{Proofs and Supporting Results}
\label{sec:pfgame1}

\subsection{Proof of Theorem \ref{thm:gamesolngen}}
\label{sec:proofdisc}

The main hurdle here is the constrained maximization over $\vz$. 
For this we use the following result, a basic application of Lagrange duality 
(from \cite{BF15}, but proved below for completeness). 
\begin{lemma}
\label{lem:gamegeng}
For any $\va \in \RR^n$,
\begin{align*}
\max_{\substack{ \vz \in [-1,1]^n , \\ \frac{1}{n} \vF \vz \geq \vb }} \;\;\frac{1}{n} \vz^\top \va 
\;=\; \min_{\sigma \geq \vzero^p} \left[ - \vb^\top \sigma + \frac{1}{n} \vnorm{\vF^\top \sigma + \va}_1 \right]
\end{align*}
\end{lemma}

With this lemma, we prove the main theorem of this paper. 

\begin{proof}[Proof of \Cref{thm:gamesolngen}]
First note that $\ell (\vz, \vg)$ is linear in $\vz$, 
$$ V = \eqref{eq:linearinz} = \frac{1}{2} \min_{\vg \in [-1,1]^n} \; \max_{\substack{ \vz \in [-1,1]^n , \\ \frac{1}{n} \vF \vz \geq \vb }} \;\; 
\frac{1}{n} \sum_{j=1}^{n} \lrb{ \ell_{+} (g_j) + \ell_{-} (g_j) + z_j \lrp{ \ell_{+} (g_j) - \ell_{-} (g_j) } } $$
Here we can rewrite the constrained maximization over $\vz$ using Lemma \ref{lem:gamegeng}: 
\begin{align}
\label{eq:innerdual}
\max_{\substack{ \vz \in [-1,1]^n , \\ \frac{1}{n} \vF \vz \geq \vb }} \;\; \frac{1}{n} \sum_{i=1}^{n} z_j \lrp{ \ell_{+} (g_j) - \ell_{-} (g_j) } 
&= \max_{\substack{ \vz \in [-1,1]^n , \\ \frac{1}{n} \vF \vz \geq \vb }} \;\; - \frac{1}{n} \vz^\top [\Gamma (\vg) ] \nonumber \\
&= \min_{\sigma \geq \vzero^p} \left[ - \vb^\top \sigma + \frac{1}{n} \vnorm{\vF^\top \sigma - \Gamma (\vg)}_1 \right]
\end{align} 

Substituting \eqref{eq:innerdual} into \eqref{eq:linearinz} and simplifying, 
\begin{align}
V 
\label{eq:gencvxaddloss}
&= \frac{1}{2} \min_{\vg \in [-1,1]^n} \; \lrb{ \frac{1}{n} \sum_{j=1}^{n} \lrb{ \ell_{+} (g_j) + \ell_{-} (g_j) } + 
\max_{\substack{ \vz \in [-1,1]^n , \\ \frac{1}{n} \vF \vz \geq \vb }} \;\;  \frac{1}{n} \sum_{j=1}^{n} z_j \lrp{ \ell_{+} (g_j) - \ell_{-} (g_j) } } \\
&= \frac{1}{2} \min_{\vg \in [-1,1]^n} \; \lrb{ \frac{1}{n} \sum_{j=1}^{n} \lrb{ \ell_{+} (g_j) + \ell_{-} (g_j) } + 
\min_{\sigma \geq \vzero^p} \left[ - \vb^\top \sigma + \frac{1}{n} \vnorm{\vF^\top \sigma - \Gamma (\vg) }_1 \right] } \nonumber \\
\label{eq:gencvxineffgame}
&= \frac{1}{2} \min_{\sigma \geq \vzero^p} \lrb{ - \vb^\top \sigma + \min_{\vg \in [-1,1]^n} 
\lrb{ \frac{1}{n} \sum_{j=1}^{n} \lrb{ \ell_{+} (g_j) + \ell_{-} (g_j) } + \frac{1}{n} \vnorm{\vF^\top \sigma - \Gamma (\vg)}_1 } } \\
\label{eq:gencvxgame}
&= \frac{1}{2} \min_{\sigma \geq \vzero^p} \lrb{ - \vb^\top \sigma + 
\frac{1}{n} \sum_{j=1}^{n} \min_{g_j \in [-1,1]} \lrb{ \ell_{+} (g_j) + \ell_{-} (g_j) + \left| \vx_j^\top \sigma - \Gamma (g_j) \right| } }
\end{align}
The absolute value breaks down into two cases, so the inner minimization's objective can be simplified:
\begin{align}
\label{eq:gencvxmmand}
\ell_{+} (g_j) + \ell_{-} (g_j) + \left| \vx_j^\top \sigma - \Gamma (g_j) \right| 
= \begin{cases} 
2 \ell_{+} (g_j) + \vx_j^\top \sigma  \qquad & \mbox{ \; if \; } \vx_j^\top \sigma \geq \Gamma (g_j) \\ 
2 \ell_{-} (g_j) - \vx_j^\top \sigma  & \mbox{ \; if \; } \vx_j^\top \sigma < \Gamma (g_j)
\end{cases}
\end{align}

Suppose $g_j$ falls in the first case, so that $\vx_j^\top \sigma \geq \Gamma (g_j)$. 
From Assumption \ref{ass:loss}, $2 \ell_{+} (g_j) + \vx_j^\top \sigma$ is decreasing in $g_j$, 
so it is minimized for the greatest $g_j^* \leq 1$ s.t. $\Gamma (g_j^*) \leq \vx_j^\top \sigma$. 
Since $\Gamma (\cdot)$ is increasing, exactly one of two subcases holds: 
\begin{enumerate}[noitemsep]
\item
$g_j^*$ is such that $\Gamma (g_j^*) = \vx_j^\top \sigma$, in which case the minimand \eqref{eq:gencvxmmand} 
is $\ell_{+} (g_j^*) + \ell_{-} (g_j^*)$
\item
$g_j^* = 1$ so that $\Gamma (g_j^*) = \Gamma (1) < \vx_j^\top \sigma$, 
in which case the minimand \eqref{eq:gencvxmmand} is $2 \ell_{+} (1) + \vx_j^\top \sigma$
\end{enumerate}

A precisely analogous argument holds if $g_j$ falls in the second case, where $\vx_j^\top \sigma < \Gamma (g_j)$. 
Putting the cases together, we have shown the form of the summand $\Psi$, piecewise over its domain, so 
\eqref{eq:gencvxgame} is equal to $ \frac{1}{2} \min_{\sigma \geq \vzero^p} \lrb{\gamma(\sigma)}$.

We have proved the dependence of $g_j^*$ on $\vx_j^\top \sigma^*$, 
where $\sigma^*$ is the minimizer of the outer minimization of \eqref{eq:gencvxgame}. 
This completes the proof.
\end{proof}

\begin{proof}[Proof of Lemma \ref{lem:gamegeng}]
We have
\begin{align}
\max_{\substack{ \vz \in [-1,1]^n , \\ \vF \vz \geq n \vb }} \;\; \frac{1}{n} \vz^\top \va 
&= \frac{1}{n} \max_{\vz \in [-1,1]^n} \; \min_{\sigma \geq \vzero^p} \;  \lrb{ \vz^\top \va + \sigma^\top (\vF \vz - n \vb) } \nonumber \\
\label{eq:duallagrange}
&\stackrel{(a)}{=} \frac{1}{n} \min_{\sigma \geq \vzero^p} \; \max_{\vz \in [-1,1]^n} \; \lrb{ \vz^\top (\va + \vF^\top \sigma) - n \vb^\top \sigma } \\
\label{eq:gameinnerdual}
&= \frac{1}{n} \min_{\sigma \geq \vzero^p} \; \lrb{ \vnorm{\va + \vF^\top \sigma}_1 - n \vb^\top \sigma } 
= \min_{\sigma \geq \vzero^p} \left[ - \vb^\top \sigma + \frac{1}{n} \vnorm{\vF^\top \sigma + \va}_1 \right]
\end{align}
where $(a)$ is by the minimax theorem (\cite{S58}). 
\end{proof}

\subsection{Other Proofs}

\begin{lemma}
\label{lem:cvxpotconds}
The function 
$\ell_{+} (\Gamma^{-1} (m)) + \ell_{-} (\Gamma^{-1} (m))$ is convex for $m \in \lrp{ \Gamma (-1) , \Gamma (1)}$ 
under any of the conditions of \Cref{lem:cvxpotential}.
\end{lemma}

\begin{proof}[Proof of Lemma \ref{lem:cvxpotconds}] 
Define $F(m) = \ell_{+} (\Gamma^{-1} (m)) + \ell_{-} (\Gamma^{-1} (m))$. 
By basic properties of the derivative, 
\begin{align}
\label{eq:ginvderiv}
\deriv{\Gamma^{-1}}{m} = \frac{1}{\Gamma' ( \Gamma^{-1} (m))} = \frac{1}{\ell_{-}' ( \Gamma^{-1} (m)) - \ell_{+}' ( \Gamma^{-1} (m))} \geq 0 
\end{align}
where the last inequality follows by Assumption \ref{ass:loss}. Therefore, by the chain rule and \eqref{eq:ginvderiv}, 
\begin{align}
\label{eq:psideriv}
F' (m) = \frac{\ell_{-}' ( \Gamma^{-1} (m)) + \ell_{+}' ( \Gamma^{-1} (m))}{\ell_{-}' ( \Gamma^{-1} (m)) - \ell_{+}' ( \Gamma^{-1} (m))} 
\end{align}
From this, we calculate $F'' (m)$, writing $\ell'_{\pm} (\Gamma^{-1} (m))$ and $\ell''_{\pm} (\Gamma^{-1} (m))$ 
as simply $\ell'_{\pm}$ and $\ell''_{\pm}$ for clarity: 
$$ F'' (m) = \underbrace{ \frac{\deriv{\Gamma^{-1}}{m}}{\lrp{\ell_{-}' ( \Gamma^{-1} (m)) - \ell_{+}' ( \Gamma^{-1} (m))}^2} }_{(a)} 
\lrb{ \lrp{\ell_{-}' - \ell_{+}' }\lrp{\ell_{-}'' + \ell_{+}''} - \lrp{\ell_{-}' + \ell_{+}' }\lrp{\ell_{-}'' - \ell_{+}''} }
$$
From \eqref{eq:ginvderiv}, observe that the term 
$(a) = \lrp{\ell_{-}' ( \Gamma^{-1} (m)) - \ell_{+}' ( \Gamma^{-1} (m))}^{-3} \geq 0$. 
Therefore, it suffices to show that the other term is $\geq 0$. 
But this is equal to 
\begin{align}
\label{eq:lastderivpos}
\lrp{\ell_{-}' - \ell_{+}' }\lrp{\ell_{-}'' + \ell_{+}''} - \lrp{\ell_{-}' + \ell_{+}' }\lrp{\ell_{-}'' - \ell_{+}''} = 2 (\ell_{-}' \ell_{+}'' - \ell_{-}'' \ell_{+}') 
\end{align} 
This proves that condition (C) of \Cref{lem:cvxpotential} is sufficient for convexity of $F$ 
(and necessary also, under \Cref{ass:loss} on the partial losses). 

We now address the other conditions of \Cref{lem:cvxpotential}. 
(A) implies (C), because by Assumption \ref{ass:loss}, 
$\ell_{-}' , \ell_{+}'' , \ell_{-}''$ are $\geq 0$ and $\ell_{+}' \leq 0$, 
so \eqref{eq:lastderivpos} is $\geq 0$ as desired. 

Finally we prove that (B) implies (C). 
If $\ell$ is proper, then it is well known (e.g. Thm. 1 of \cite{RW10}, and \cite{BSS05}) that for all $x \in (-1,1)$,
$$ \frac{\ell_{-}' (x)}{1+x} = - \frac{\ell_{+}' (x)}{1-x} $$
(This is a simple and direct consequence of stationary conditions from the properness definition.)

Define the function $w(x) = \frac{\ell_{-}' (x)}{1+x} = - \frac{\ell_{+}' (x)}{1-x}$; 
we drop the argument and write it and its derivative as $w$ and $w'$ for clarity. 
By direct computation,  
\begin{align*}
\ell_{-}' \ell_{+}'' - \ell_{-}'' \ell_{+}' 
&= \lrb{ (1+x) w \lrp{w + (x-1) w'} } - \lrb{ (w + (1+x) w') (x-1) w } \\
&= \lrb{ (1+x) w^2 + (x^2 - 1) w w' } - \lrb{ (x-1) w^2 + (x^2 - 1) w w' } = 2 w^2 \geq 0
\end{align*}
so (C) is true as desired.
\end{proof}

\begin{proof}[Proof of Lemma \ref{lem:cvxpotential}]
Continuity follows by checking $\Psi(m)$ at $m = \pm 1$. 

For Lipschitzness, note that for $m \in \lrp{ \Gamma (-1) , \Gamma (1)}$, by \eqref{eq:psideriv}, 
\begin{align}
\Psi' (m)
&= \frac{\ell_{-}' ( \Gamma^{-1} (m)) + \ell_{+}' ( \Gamma^{-1} (m))}{\ell_{-}' ( \Gamma^{-1} (m)) - \ell_{+}' ( \Gamma^{-1} (m))} \\
\label{eq:psiderivlb}
&= -1 + \frac{2 \ell_{-}' ( \Gamma^{-1} (m)) }{\ell_{-}' ( \Gamma^{-1} (m)) - \ell_{+}' ( \Gamma^{-1} (m))} \\
\label{eq:psiderivub}
&= 1 - \frac{2 (- \ell_{+}' ( \Gamma^{-1} (m)) ) }{\ell_{-}' ( \Gamma^{-1} (m)) - \ell_{+}' ( \Gamma^{-1} (m))} 
\end{align}
Using Assumption \ref{ass:loss} on the partial losses, 
equations \eqref{eq:psiderivlb} and \eqref{eq:psiderivub} respectively make clear that $\Psi' (m) \geq -1$ 
and $\Psi' (m) \leq 1$ on this interval.
Since $\Psi' (m)$ is $-1$ for $m < \Gamma (-1)$ 
and $1$ for $m > \Gamma (1)$, it is 1-Lipschitz. 

As for convexity, since $\Psi$ is linear outside the interval $\lrp{ \Gamma (-1) , \Gamma (1)}$, 
it suffices to show that $\Psi(m)$ is convex inside this interval, which is shown in \Cref{lem:cvxpotconds}. 
\end{proof}

\subsection{Results and Proofs from \Cref{sec:extensions}}

\begin{proof}[Proof of \Cref{thm:weighteval}]
The proof is similar to that of \Cref{thm:gamesolngen}, which it generalizes.
First note that 
\begin{align}
\label{eq:weightevalduality}
\max_{\substack{ \vz \in [-1,1]^n , \\ \frac{1}{n} \vF \vz \geq \vb }} \;\; \frac{1}{n} \sum_{i=1}^{n} r_j z_j \lrp{ \ell_{+} (g_j) - \ell_{-} (g_j) } 
&= \max_{\substack{ \vz \in [-1,1]^n , \\ \frac{1}{n} \vF \vz \geq \vb }} \;\; - \frac{1}{n} \vz^\top [\vr \circ \Gamma (\vg) ] \nonumber \\
&= \min_{\sigma \geq \vzero^p} \left[ - \vb^\top \sigma + \frac{1}{n} \vnorm{\vF^\top \sigma - (\vr \circ \Gamma (\vg)) }_1 \right]
\end{align}
where the last equality uses \Cref{lem:gamegeng}. 

Therefore, using \eqref{eq:weightevalduality} on the left-hand side of what we wish to prove, 
\begin{align}
\label{eq:genweightevalgame}
V 
&= \frac{1}{2} \min_{\vg \in [-1,1]^n} \; \lrb{ \frac{1}{n} \sum_{j=1}^{n} r_j \lrb{ \ell_{+} (g_j) + \ell_{-} (g_j) } + 
\max_{\substack{ \vz \in [-1,1]^n , \\ \frac{1}{n} \vF \vz \geq \vb }} \;\;  \frac{1}{n} \sum_{i=1}^{n} r_j z_j \lrp{ \ell_{+} (g_j) - \ell_{-} (g_j) } } \nonumber \\
&= \frac{1}{2} \min_{\vg \in [-1,1]^n} \; \lrb{ \frac{1}{n} \sum_{j=1}^{n} r_j \lrb{ \ell_{+} (g_j) + \ell_{-} (g_j) } + 
\min_{\sigma \geq \vzero^p} \left[ - \vb^\top \sigma + \frac{1}{n} \sum_{j=1}^n \abs{\vx_j^\top \sigma - r_j \Gamma (g_j)} \right] } \nonumber \\
&= \frac{1}{2} \min_{\sigma \geq \vzero^p} \lrb{ - \vb^\top \sigma + \frac{1}{n} \sum_{j=1}^{n} 
\min_{g_j \in [-1,1]} \lrp{ r_j \lrb{ \ell_{+} (g_j) + \ell_{-} (g_j) } + \abs{\vx_j^\top \sigma - r_j \Gamma (g_j)}} }
\end{align}

As in the proof of \Cref{thm:gamesolngen}, the inner minimization's objective can be simplified:
\begin{align}
\label{eq:gencvxmmandwt}
r_j (\ell_{+} (g_j) + \ell_{-} (g_j)) + \left| \vx_j^\top \sigma - r_j \Gamma (g_j) \right| 
= \begin{cases} 
2 r_j \ell_{+} (g_j) + \vx_j^\top \sigma  \qquad & \mbox{ \; if \; } \vx_j^\top \sigma \geq r_j \Gamma (g_j) \\ 
2 r_j \ell_{-} (g_j) - \vx_j^\top \sigma  & \mbox{ \; if \; } \vx_j^\top \sigma < r_j \Gamma (g_j)
\end{cases}
\end{align}

Suppose $g_j$ falls in the first case, so that $\vx_j^\top \sigma \geq r_j \Gamma (g_j)$. 
From Assumption \ref{ass:loss}, $2 r_j \ell_{+} (g_j) + \vx_j^\top \sigma$ is decreasing in $g_j$, 
so it is minimized for the greatest $g_j^* \leq 1$ s.t. $\Gamma (g_j^*) \leq \frac{\vx_j^\top \sigma}{r_j }$. 
Since $\Gamma (\cdot)$ is increasing, exactly one of two subcases holds: 
\begin{enumerate}[a)]
\item
$g_j^*$ is such that $\Gamma (g_j^*) = \frac{\vx_j^\top \sigma}{r_j }$, in which case the minimand \eqref{eq:gencvxmmandwt} 
is $r_j (\ell_{+} (g_j^*) + \ell_{-} (g_j^*))$
\item
$g_j^* = 1$ so that $\Gamma (g_j^*) = \Gamma (1) < \frac{\vx_j^\top \sigma}{r_j }$, 
in which case the minimand \eqref{eq:gencvxmmandwt} is $2 r_j \ell_{+} (1) + \vx_j^\top \sigma$
\end{enumerate}

A precisely analogous argument holds if $g_j$ falls in the second case, where $\vx_j^\top \sigma < \Gamma (g_j)$. 
So as before, we have proved the dependence of $g_j^*$ on $\vx_j^\top \sigma^*$, 
where $\sigma^*$ is the minimizer of the outer minimization of \eqref{eq:genweightevalgame}. 
This completes the proof.
\end{proof}

%\begin{lemma}
%\label{lem:labelnoiseinnerdual}
%For any $\va \in \RR^n$ and $\vr \geq \vzero^n$,
%\begin{align*}
%\max_{\substack{ -\vr \leq \vz \leq \vr , \\ \frac{1}{n} \vF \vz \geq \vb }} \;\frac{1}{n} \vz^\top \va 
%\;=\; \min_{\sigma \geq \vzero^p} \left[ - \vb^\top \sigma + \frac{1}{n} \vr^\top \abs{\vF^\top \sigma + \va} \right]
%\end{align*}
%\end{lemma}
%\begin{proof}[Proof of Lemma \ref{lem:labelnoiseinnerdual}]
%The proof almost exactly follows that of Lemma \ref{lem:gamegeng}, except for the setting of the optimal $\vz^*$ in Eq. \eqref{eq:gameinnerdual}.
%\end{proof}

\begin{theorem}
\label{thm:gamegendev}
For any $\vc \geq \vzero^n$, 
\begin{align*}
\min_{\vg \in [-1,1]^n} \; \max_{\substack{ \vz \in [-1,1]^n , \\ \abs{\frac{1}{n} \vF \vz - \vb} \leq \vc }} \;\ell (\vz, \vg)
\;=\; \min_{\sigma \in \RR^p} \left[ - \vb^\top \sigma + \frac{1}{n} \sum_{j=1}^n \Psi ( \vx_{j}^\top \sigma ) + \vc^\top \abs{\sigma} \right]
\end{align*}
Writing $\sigma_{\vc}^*$ as the minimizer of the RHS above, 
the optimal predictions $\vg^* = \vg (\sigma_{\vc}^*)$, as in \Cref{thm:gamesolngen}. 
%Let $\sigma_{\vc}^*$ be the minimizer of the right-hand side above.
%Then the optimal $\vg^* = \vg (\sigma_{\vc}^*)$, the same function of the optimal weighting as in \eqref{eq:gipredform}.
\end{theorem} 
\begin{proof}[Proof of Theorem \ref{thm:gamegendev}]
This is proved exactly like \Cref{thm:gamesolngen}, but using Lemma \ref{lem:weightednorminnerdual} instead of Lemma \ref{lem:gamegeng} in that proof.
\end{proof}

\begin{lemma}
\label{lem:weightednorminnerdual}
For any $\va \in \RR^n$ and $\vc \geq \vzero^n$,
\begin{align*}
\max_{\substack{ \vz \in [-1,1]^n , \\ \abs{\frac{1}{n} \vF \vz - \vb} \leq \vc }} \;\frac{1}{n} \vz^\top \va 
\;=\; \min_{\sigma \in \RR^p} \left[ - \vb^\top \sigma + \frac{1}{n} \vnorm{\vF^\top \sigma + \va}_1 + \vc^\top \abs{\sigma} \right]
\end{align*}
\end{lemma}
\begin{proof}
\begin{align*}
\max_{\substack{ \vz \in [-1,1]^n , \\ \abs{\frac{1}{n} \vF \vz - \vb} \leq \vc }} \;\frac{1}{n} \vz^\top \va 
&= \max_{\substack{ \vz \in [-1,1]^n , \\ \frac{1}{n} \vF \vz - \vb \leq \vc , \\ - \frac{1}{n} \vF \vz + \vb \leq \vc }} \;\frac{1}{n} \vz^\top \va \\
&= \frac{1}{n} \;\max_{\vz \in [-1,1]^n} \min_{\lambda, \xi \geq \vzero^p} \lrb{ \vz^\top \va 
+ \lambda^\top (- \vF \vz + n \vb + n \vc) + \xi^\top (\vF \vz - n \vb + n \vc) } \\
&= \frac{1}{n} \;\min_{\lambda, \xi \geq \vzero^p} \max_{\vz \in [-1,1]^n} \lrb{ \vz^\top 
(\va + \vF^\top (\xi - \lambda)) + \lambda^\top ( n \vb + n \vc) + \xi^\top (- n \vb + n \vc) } \\
&= \frac{1}{n} \;\min_{\lambda, \xi \geq \vzero^p} \lrb{  
\vnorm{\va + \vF^\top (\xi - \lambda)}_1 - n \vb^\top (\xi - \lambda) + n \vc^{\top} (\xi + \lambda) }
\end{align*}
where the interchanging of min and max is again justified by the minimax theorem (\cite{S58}), 
since the objective is linear in each variable and one of the constraint sets is closed. 

Suppose for some $j \in [n]$ that $\xi_j > 0$ and $\lambda_j > 0$. 
Then subtracting $\min(\xi_j, \lambda_j)$ from both does not affect the value $[\xi - \lambda]_{j}$, 
but always decreases $[\xi + \lambda]_{j}$, and therefore always decreases the objective function. 
Therefore, we can w.l.o.g. assume that $\forall j \in [n]: \min(\xi_j, \lambda_j) = 0$. 
Defining $\sigma_j = \xi_j - \lambda_j$ for all $j$ (so that $\xi_j = [\sigma_j]_+$ and $\lambda_j = [\sigma_j]_-$), 
the last equality above becomes
\begin{align*}
\frac{1}{n} \;\min_{\lambda, \xi \geq \vzero^p} &\lrb{  
\vnorm{\va + \vF^\top (\xi - \lambda)}_1 - n \vb^\top (\xi - \lambda) + n \vc^{\top} (\xi + \lambda) } \\
&= \frac{1}{n} \;\min_{\sigma \in \RR^p} \lrb{  
\vnorm{\va + \vF^\top \sigma}_1 - n \vb^\top \sigma + n \vc^\top \abs{\sigma} }
\end{align*}
\end{proof}

\end{document}